\documentclass[10pt,a4paper]{article}
\usepackage[T1]{fontenc}
\usepackage{lmodern}
\usepackage[margin=24mm,top=21mm,bottom=22mm]{geometry}
\usepackage{amsmath,amssymb,amsthm,booktabs,microtype}
\usepackage[numbers,sort&compress]{natbib}
\usepackage[hidelinks]{hyperref}
\hypersetup{pdftitle={RoPE attention is an exact forward-pass gradient step with softmax intact},
 pdfauthor={Julie Huang, Maggie Chlon, Leon Chlon},
 pdfsubject={Exact effective-weight representation of RoPE--softmax attention}}
\newtheorem{theorem}{Theorem}
\newtheorem{proposition}[theorem]{Proposition}
\newtheorem{corollary}[theorem]{Corollary}
\newcommand{\R}{\mathbb{R}}
\newcommand{\T}{^{\!\top}}
\newcommand{\Iset}{\mathcal{I}}
\newcommand{\DM}{\Delta M}
\newcommand{\norm}[1]{\left\lVert #1\right\rVert}
\newcommand{\dd}{\mathrm{d}}
\begin{document}
\begin{center}
{\Large\bfseries RoPE attention is an exact forward-pass\\[2pt]
gradient step with softmax intact.\par}
\vspace{9pt}
{\large Julie Huang\textsuperscript{1}, Maggie Chlon\textsuperscript{1},
Leon Chlon\textsuperscript{1,2,*}\par}
\vspace{4pt}
\textsuperscript{1}Hassana Labs\quad
\textsuperscript{2}Department of Information Engineering, Oxford University\par
\textsuperscript{*}Correspondence:
\href{mailto:leochlon@robots.ox.ac.uk}{leochlon@robots.ox.ac.uk}
\end{center}
\vspace{2pt}
\begin{abstract}
\noindent We derive an exact gradient-step representation of the RoPE--softmax forward pass.
For every deterministic RoPE--softmax attention head with arbitrary affine projection
weights, we construct a query-dependent effective matrix $\DM_i$ satisfying
$y_i=\mu_i+u_i\T\DM_i$, where $\mu_i$ is the uniform mean of the attended values
and $u_i$ is the augmented query input. The construction applies the classical
exponential divided difference $\rho=\varphi_1$ to retain softmax exactly. Its positive
coefficients give a unit gradient-step representation on a query-conditioned quadratic
objective. The same function connects the RoPE generator to exact positional finite
differences. We derive a tokenwise formula for the error of reusing one query's matrix
and prove that a nonconstant finite-cache head cannot admit a globally exact affine
query readout. Reconstruction checks and frozen-reuse calibration on one pretrained
Qwen2.5-0.5B layer verify the representation and quantify the correction required
when one query's matrix is reused.
\end{abstract}

\section{Contribution and relation to existing constructions}
We apply the exponential divided difference to express full RoPE--softmax attention
in the original projection-weight coordinates, including biases. This yields an exact
query-conditioned matrix, its per-token decomposition, and the corresponding
quadratic gradient-step representation. A boundedness proposition characterises the
obstruction to a globally shared affine readout; a matrix-function identity connects
the same construction to positional derivatives and finite differences.

The construction uses the classical $\varphi_1$ function from exponential integrators
\cite{hochbruck2010} to factorise the softmax readout. Softmax attention is a
Nadaraya--Watson-type normalised kernel average
\cite{nadaraya1964,watson1964}: with $\kappa_{ij}=\exp(s_{ij})$, its output is
$\sum_j\kappa_{ij}v_j/\sum_j\kappa_{ij}$. RoPE enters $s_{ij}$ through the relative
rotation inside the exponential.

Von Oswald et al.\ \cite{vonoswald2023} construct a least-squares gradient step using
linear self-attention. Schlag et al.\ \cite{schlag2021} develop the fast-weight view of
linear attention, and Dai et al.\ \cite{dai2023} use a linearised attention expression
to obtain an implicit weight update. Here the coefficients retain the exponential and
its normalisation for arbitrary projection weights. Modern Hopfield networks
give an exact softmax-attention correspondence \cite{ramsauer2021}; the object derived
here is a query-conditioned effective matrix in input coordinates, together with its
score-shift family and exact reuse residual.

\section{Exact effective weights with softmax intact}
\subsection{RoPE, affine projections, and the attended bank}
Let $x_i\in\R^m$ be a column input, with query/key width $d$ and value width $r$.
Use row-vector projections
\begin{equation}
q_i=x_i\T W_q+b_q,\qquad
k_j=x_j\T W_k+b_k,\qquad
v_j=x_j\T W_v+b_v,
\label{eq:projections}
\end{equation}
where $W_q,W_k\in\R^{m\times d}$, $W_v\in\R^{m\times r}$, and biases are rows.
For standard RoPE with fixed frequencies \cite{su2021}, set $\tau=\sqrt d$ and
\begin{equation}
R(p)=e^{pA},\qquad
A=\operatorname{diag}_{\ell=1}^{d/2}
\begin{pmatrix}0&-\omega_\ell\\ \omega_\ell&0\end{pmatrix},\qquad
R(p_i)\T R(p_j)=R(p_j-p_i).
\label{eq:rotation}
\end{equation}
The rotated rows are $\widetilde q_i=q_iR(p_i)\T$ and
$\widetilde k_j=k_jR(p_j)\T$. Augment the query input and projection by
\begin{equation}
u_i=\begin{pmatrix}x_i\\1\end{pmatrix},\qquad
\widehat W_q=\begin{pmatrix}W_q\\b_q\end{pmatrix},\qquad
L_i=\frac{\widehat W_qR(p_i)\T}{\tau},\qquad
a_{ij}=L_i\widetilde k_j\T.
\label{eq:features}
\end{equation}
Then the raw score is exactly $s_{ij}=\widetilde q_i\widetilde k_j\T/\tau=u_i\T a_{ij}$.
In the bias-free case this is
$s_{ij}=x_i\T W_qR(p_j-p_i)W_k\T x_j/\tau$.
Let $\Iset_i$ be a finite nonempty set of permitted keys; causal masking selects this
set. For deterministic attention, define
\begin{equation}
N_i=|\Iset_i|,\qquad Z_i=\sum_{j\in\Iset_i}e^{s_{ij}},\qquad
\mu_i=\frac1{N_i}\sum_{j\in\Iset_i}v_j,\qquad
y_i=\frac1{Z_i}\sum_{j\in\Iset_i}e^{s_{ij}}v_j.
\label{eq:attention}
\end{equation}
Here $\mu_i$ is the uniform arithmetic mean of the permitted values.

\subsection{The exponential divided difference}
The entire function
\begin{equation}
\rho(s)=\varphi_1(s)=\exp[0,s]
=\int_0^1e^{ts}\,\dd t
=\sum_{n=0}^{\infty}\frac{s^n}{(n+1)!}
=\begin{cases}(e^s-1)/s,&s\ne0,\\1,&s=0\end{cases}
\label{eq:rho}
\end{equation}
is the first divided difference of the exponential at $0$ and $s$, with its continuous
value at coincident arguments. It is the standard $\varphi_1$ of exponential
integration \cite{hochbruck2010}; $\rho(s)>0$ for real $s$ and $e^s-1=s\rho(s)$.

\begin{theorem}[Exact effective-weight representation]
For the head in \eqref{eq:projections}--\eqref{eq:attention}, define
\begin{equation}
\boxed{\quad c_{ij}=\frac{\rho(s_{ij})}{Z_i},\qquad
\DM_i=\sum_{j\in\Iset_i}c_{ij}a_{ij}(v_j-\mu_i),\qquad
y_i=\mu_i+u_i\T\DM_i.\quad}
\label{eq:identity}
\end{equation}
The last equality holds exactly for every query and all finite projection weights
and inputs.
\end{theorem}
\begin{proof}
The centred values sum to zero. Therefore
\begin{align*}
y_i-\mu_i
&=Z_i^{-1}\sum_j e^{s_{ij}}(v_j-\mu_i)
 =Z_i^{-1}\sum_j(e^{s_{ij}}-1)(v_j-\mu_i)\\
&=Z_i^{-1}\sum_j s_{ij}\rho(s_{ij})(v_j-\mu_i)
 =u_i\T\sum_j c_{ij}a_{ij}(v_j-\mu_i).
\end{align*}
\end{proof}
In particular, without projection biases the effective matrix in the original weights is
\begin{equation}
\DM_i=\frac1{\tau Z_i}\sum_{j\in\Iset_i}\rho(s_{ij})
\bigl[W_qR(p_j-p_i)W_k\T x_j\bigr]\bigl[x_j\T W_v-\mu_i\bigr].
\label{eq:original}
\end{equation}
At a fixed bank and query position, write $L=L_i$ and $a_j=a_{ij}$. Then
$\DM_i=LT_i$, where
$T_i=\sum_j c_{ij}\widetilde k_j\T(v_j-\mu)\in\R^{d\times r}$.
The cache supplies the outer-product directions; the query selects their positive
coefficients. This gives $\operatorname{rank}(\DM_i)\le d$ in this representation.

\begin{corollary}[A query-conditioned unit gradient step]
For each fixed query, define $B\in\R^{(m+1)\times r}$ and
\begin{equation}
\mathcal E_i(B)=\frac12\sum_{j\in\Iset_i}c_{ij}
\norm{a_{ij}\T B-(v_j-\mu_i)}_2^2.
\quad\text{Then}\quad
0-\nabla_B\mathcal E_i(0)=\DM_i.
\label{eq:gradient}
\end{equation}
\end{corollary}
\begin{proof}
With the query, features, centred values, and coefficients held fixed,
$\nabla_B\mathcal E_i(B)=\sum_j c_{ij}a_{ij}(a_{ij}\T B-(v_j-\mu_i))$;
evaluation at $B=0$ gives the result.
\end{proof}
A unit step from zero on $\mathcal E_i$ gives the effective matrix and its exact
readout $y_i=\mu_i+u_i\T\DM_i$. Every coefficient is computed directly from
the query and its attended bank through \eqref{eq:identity}.
As all scores approach zero, $c_{ij}\to1/N_i$; the first-order readout is
$\mu_i+N_i^{-1}\sum_j s_{ij}(v_j-\mu_i)$. This is the centred outer-product form
underlying linear-attention gradient constructions. The corresponding regression
step is specified by the feature/label encoding, projection choices, and step scale
of von Oswald et al.\ \cite{vonoswald2023}.

\section{The same function links RoPE derivatives and finite differences}
For fixed content $z_0$, let $z(p)=R(p)z_0$ be a rotated column feature. The generator
in \eqref{eq:rotation} gives
\begin{equation}
\partial_p^n z(p)=A^n z(p),\qquad n\ge1.
\label{eq:derivatives}
\end{equation}
The matrix extension of \eqref{eq:rho} is defined by the convergent power series
$\rho(B)=\sum_{n\ge0}B^n/(n+1)!$. Consequently
\begin{equation}
\boxed{\quad e^{\delta A}-I=\delta A\rho(\delta A),\qquad
z(p+\delta)-z(p)=\delta A\rho(\delta A)z(p).\quad}
\label{eq:matrixrho}
\end{equation}
The power-series definition also covers zero-frequency blocks. These identities expose
RoPE as a bank of position-response operators on frozen features: $A$ gives the
infinitesimal response and $A\rho(\delta A)$ gives the exact finite-difference quotient.
The same entire function converts the scalar softmax increment $e^s-1$ into
$s\rho(s)$ and the matrix rotary increment $e^{\delta A}-I$ into
$\delta A\rho(\delta A)$. In \eqref{eq:original}, rotary position structure enters
through $R(p_j-p_i)$ and its scores, while $\rho(s_{ij})$ retains their full exponential
response. The derivatives in \eqref{eq:derivatives} are with respect to position at fixed content.

\section{Residual coordinates, gauge, and exact reuse error}
\subsection{Residual write and score-shift family}
In the bias-free case, if the head input is also the residual input and
$W_o\in\R^{r\times m}$, then
\begin{equation}
x_i\T+y_iW_o=x_i\T(I_m+\DM_iW_o)+\mu_iW_o.
\label{eq:residual}
\end{equation}
For a pre-normalised head, use $x_i=\operatorname{Norm}(h_i)$ and the exact residual
$h_i\T+(\mu_i+u_i\T\DM_i)W_o+b_o$; projected head writes are summed for multi-head
attention. Equation~\eqref{eq:identity} is a secant representation of the readout.
At a fixed bank and position, its differential is
$\dd y_i=\dd u_i\T\DM_i+u_i\T\dd\DM_i$; the second term contains the
query derivatives of the coefficients.

Effective matrices depend on a score-shift convention. For any fixed
$g\in\R^{m+1}$, put $\sigma_i=u_i\T g$. Softmax is unchanged by subtracting
$\sigma_i$ from every attended score. Applying Theorem~1 to the shifted features gives
\begin{equation}
\begin{aligned}
a_{ij}^g&=a_{ij}-g,&
c_{ij}^g&=\frac{e^{\sigma_i}\rho(s_{ij}-\sigma_i)}{Z_i},\\
\DM_i^g&=\sum_j c_{ij}^g a_{ij}^g(v_j-\mu_i),&
y_i&=\mu_i+u_i\T\DM_i^g.
\end{aligned}
\label{eq:gauge}
\end{equation}
Key shifts give the subfamily $g=L_i b$, of dimension at most $d$ at a fixed position.
We call \eqref{eq:identity} the \emph{original-score gauge}; all matrix-transfer results
below use this gauge. The displayed transformations give a family of exact matrices
for the same output map.

\subsection{Global affine representability}
\begin{proposition}[Boundedness obstruction to a global affine readout]
Fix a finite nonempty key/value bank and query position, and let
$F:\R^m\to\R^r$ be its softmax head output as the query coordinate ranges over
all of $\R^m$. If $F$ is nonconstant, no fixed $b$ and $M$ satisfy
$F(x)=b+x\T M$ for every $x\in\R^m$.
\label{prop:bounded}
\end{proposition}
\begin{proof}
Every output is a convex combination of the fixed values, so
$\norm{F(x)}_2\le\max_j\norm{v_j}_2$ for all $x$.
If $M\ne0$, choose $z$ with $z\T M\ne0$; then $\norm{b+t z\T M}_2\to\infty$
as $|t|\to\infty$. Hence $M=0$, which makes $F$ constant, a contradiction.
\end{proof}
Matrix transfer on a finite query set is quantified by the following exact residual.

\subsection{The exact term omitted by freezing}
Fix the bank and query position so that $a_j$ and $\mu$ are shared, and select an
anchor query $a$. Define $y_i^{\mathrm{fr}}=\mu+u_i\T\DM_a$. Then
\begin{equation}
\boxed{\quad
e_i=y_i-y_i^{\mathrm{fr}}
=u_i\T(\DM_i-\DM_a)
=\sum_{j\in\Iset}(c_{ij}-c_{aj})(u_i\T a_j)(v_j-\mu).
\quad}
\label{eq:tokenerror}
\end{equation}
Thus $y_i=\mu+u_i\T\DM_a+e_i$, and $e_iW_o$ is the exact projected correction.
Each summand resolves this correction by cached token. In bias-free coordinates,
freezing anchor $0$ omits exactly $x_i\T(\DM_i-\DM_0)$, with
\begin{equation}
\norm{e_i}_2\le\norm{u_i}_2\norm{\DM_i-\DM_a}_{\mathrm{op}}.
\label{eq:bound}
\end{equation}
The fixed-position condition makes the features shared; changing the position also
changes $L_i$, and changing the bank can change both features and mean.

\section{Pretrained reconstruction and frozen-reuse calibration}
\subsection{Protocol and reconstruction precision}
We evaluate Qwen2.5-0.5B \cite{qwen2024}, layer 23 (zero-indexed), on $C=16$ article
prefixes of 511 tokens, all 14 query heads, and $Q=32$ alternative one-token
continuations at position 511. Continuation $0$ is the observed next token; the other
31 are the model's highest-ranked distinct nonspecial alternatives. Each continuation
reads the same prefix bank. Projection biases and the model's pre-normalised head
inputs are included. Reconstruction is checked both for the prefix-only readout and
for the full causal readout including the continuation's own key/value.

For the latter, if the self token receives mass $\pi_i$, the exact relation is
$y_i^{\mathrm{full}}=(1-\pi_i)y_i^{\mathrm{prefix}}+\pi_i v_i^{\mathrm{self}}$.
The full-bank version of \eqref{eq:identity} uses its own scores and uniform mean.
The 240 reconstruction rows comprise 224 individual head/context pairs and 16
combined projected writes. All pass the reported checks. Combined prefix and full
reconstruction MSEs in float64 are approximately $9.8\times10^{-30}$ and
$9.6\times10^{-30}$. Against the native FP32 combined write, the MSE is
$7.01\times10^{-14}$ and the maximum absolute difference is $1.4\times10^{-5}$.
The float64 figures measure agreement between algebraically equivalent evaluations;
the native comparison measures agreement at the model's execution precision.

\subsection{Frozen reuse as calibration of the omitted term}
The calibration evaluates the residual in \eqref{eq:tokenerror} on alternative
queries from one position. Let $o_{ci}=\sum_h y_{cih}W_{oh}\in\R^D$ be the exact
combined projected prefix write and let $f_{ci}^{(a)}$ denote a predictor formed
using anchor $a$. Define
\begin{equation}
E_a(f)=\frac1C\sum_{c=1}^C\frac1{(Q-1)D}
\sum_{i\ne a}\norm{f_{ci}^{(a)}-o_{ci}}_2^2,
\qquad
\mathrm{rRMSE}_a(f)=\sqrt{\frac{E_a(f)}{E_a(\mathrm{constant})}},
\label{eq:metrics}
\end{equation}
where $f_{ci}^{(a),\mathrm{constant}}=o_{ca}$. All-anchor evaluation uses
$\overline E(f)=Q^{-1}\sum_{a=0}^{Q-1}E_a(f)$ and
$\overline{\mathrm{rRMSE}}(f)=\sqrt{\overline E(f)/\overline E(\mathrm{constant})}$.
Squared errors are pooled with equal context weights, with ratios taken after pooling.
The combined write sums projected heads before its error is measured.

In addition to constant-anchor and frozen-matrix predictors, the per-head controls are
\begin{equation}
y_i^{\mu}=\mu,\qquad
y_i^{\mathrm{lin}}=\mu+\frac1N\sum_j s_{ij}(v_j-\mu).
\label{eq:controls}
\end{equation}
The second is the first-order softmax expansion at zero logits.
Every control undergoes the same output projection and
aggregation as its corresponding exact head.

\clearpage
\begin{table}[ht]
\centering
\caption{Frozen-reuse calibration: pooled MSE of the combined projected prefix write.
Anchor 0 is the observed next token. All anchors averages the 32 possible anchors
within each of the 16 contexts. Matrices use the original-score gauge.}
\label{tab:calibration}
\begin{tabular}{lrr}
\toprule
Predictor & Anchor 0 & All anchors\\
\midrule
Constant anchor output & 0.047787 & 0.045672\\
Cache mean ($\DM=0$) & 0.124573 & 0.124513\\
Frozen anchor matrix & 0.324970 & 0.356763\\
Linearised softmax & 0.161842 & 0.162068\\
\bottomrule
\end{tabular}
\end{table}

The target RMS is $0.670545$ for anchor 0 and $0.670824$ across all anchors.
Frozen rRMSE against the constant-anchor output is $2.608$ for anchor 0 and $2.795$
across all anchors. The anchor-0 RMSE ratio against the cache mean is $1.615$.
At the individual-head level, the all-anchor rRMSE is $0.728$--$0.933$ for 13 of
the 14 heads, and $8.516$ for head 5. Each per-head ratio uses that head's
constant-anchor output as its baseline.

The frozen-matrix row measures the MSE of the omitted correction in
\eqref{eq:tokenerror}. Adding that tokenwise correction to the reused matrix's
readout recovers the exact attention output for each query.

\paragraph{Code and archive.}
GitHub repository: \url{https://github.com/leochlon/secant/tree/main}.\\
Zenodo record: \url{https://zenodo.org/records/22543637}.

\bibliographystyle{unsrtnat}
{\small\bibliography{references}}
\end{document}